\documentclass[10pt,twocolumn,letterpaper]{article}

\usepackage[pagenumbers]{iccv}      

\definecolor{ForestGreen}{cmyk}{0.864, 0.0, 0.429, 0.396}

\usepackage{amsmath,amsfonts,bm}

\usepackage{amsthm}
\theoremstyle{definition}

\newtheorem{proposition}{Proposition}
\newtheorem{corollary}{Corollary}

\def\Figref#1{Figure~\ref{#1}}

\def\Secref#1{Section~\ref{#1}}

\def\eqref#1{equation~\ref{#1}}

\def\Algref#1{Algorithm~\ref{#1}}

\def\1{\bm{1}}

\DeclareMathAlphabet{\mathsfit}{\encodingdefault}{\sfdefault}{m}{sl}
\SetMathAlphabet{\mathsfit}{bold}{\encodingdefault}{\sfdefault}{bx}{n}

\def\gA{{\mathcal{A}}}

\def\gL{{\mathcal{L}}}

\def\gS{{\mathcal{S}}}

\def\gV{{\mathcal{V}}}

\newcommand{\R}{\mathbb{R}}

\DeclareMathOperator*{\argmin}{arg\,min}

\usepackage{multirow}

\usepackage[ruled]{algorithm2e}
\usepackage{algpseudocode}
\usepackage{adjustbox}

\newcommand\shortsection[1]{\vspace{4pt}{\noindent\bf #1}.}

\usepackage{comment}

\definecolor{iccvblue}{rgb}{0.21,0.49,0.74}
\usepackage[pagebackref,breaklinks,colorlinks,allcolors=iccvblue]{hyperref}

\title{Generalizable Targeted Data Poisoning against Varying Physical Objects}

\author{Zhizhen Chen$^1$\quad Zhengyu Zhao$^1$\quad Subrat Kishore Dutta$^{2}$\\ Chenhao Lin$^1$\quad Chao Shen$^1$\quad Xiao Zhang$^2$\\
$^1$Xi'an Jiaotong University \quad 
$^2$CISPA Helmholtz Center for Information Security\\
{\tt\small zhizhenc@stu.xjtu.edu.cn\quad \{subrat.dutta, xiao.zhang\}@cispa.de} \\
{\tt\small \{zhengyu.zhao, linchenhao\}@xjtu.edu.cn\quad chaoshen@mail.xjtu.edu.cn}
}

\begin{document}

\maketitle

\begin{abstract}
Targeted data poisoning (TDP) aims to compromise the model's prediction on a specific (test) target by perturbing a small subset of training data. 
Existing work on TDP has focused on an overly ideal threat model in which the same image sample of the target is used during both poisoning and inference stages.
However, in the real world, a target object often appears in complex variations due to changes of physical settings such as viewpoint, background, and lighting conditions.
In this work, we take the first step toward understanding the real-world threats of TDP by studying its generalizability across varying physical conditions.
In particular, we observe that solely optimizing gradient directions, as adopted by the best previous TDP method, achieves limited generalization.
To address this limitation, we propose optimizing both the gradient direction and magnitude for more generalizable gradient matching, thereby leading to higher poisoning success rates.
For instance, our method outperforms the state of the art by $19.49\%$  when poisoning CIFAR-10 images targeting multi-view cars.
\end{abstract}

\section{Introduction}
Training modern machine learning models requires a large amount of data, which is often crawled from the internet \citep{schuhmann2022laion, gadre2024datacomp, sharma2018conceptual, changpinyo2021conceptual}.
However, the data published by untrusted third parties may be collected as part of the training set \citep{carlini2024poisoning}, amplifying the threats of data poisoning~\citep{biggio2012poisoning, xiao2012adversarial, 8685687, yao2019latent, koh2017understanding, munoz2017towards}.
Among different types of data poisoning~\citep{cina2023wild}, \emph{targeted data poisoning} (TDP) aims to compromise the model's behavior solely on specific targets without affecting the model's overall accuracy.
In particular, clean-label TDP, which modifies training samples by injecting small, imperceptible pixel perturbations~\cite{shafahi2018poison, huang2020metapoison, aghakhani2021bullseye, zhu2019transferable, geiping2021witches}, is even more stealthy.

Although clean-label TDP has been extensively studied in the existing literature~\citep{huang2020metapoison, aghakhani2021bullseye, geiping2021witches}, poisoning against a physical object is still an open problem~\citep{goldblum2022dataset, cina2023wild}.
These works typically focus on an ideal threat model where the poisoner uses a single image sample of the target in the poisoning stage, and the \textit{same} sample is evaluated during inference.
However, due to background, lighting, and viewpoint changes, a physical object often appears in its complex variants in the real world.
Therefore, using a single image sample is insufficient to capture the real-world threats of TDP.

\begin{figure}[!t]
\centering
\includegraphics[width=0.95\linewidth]{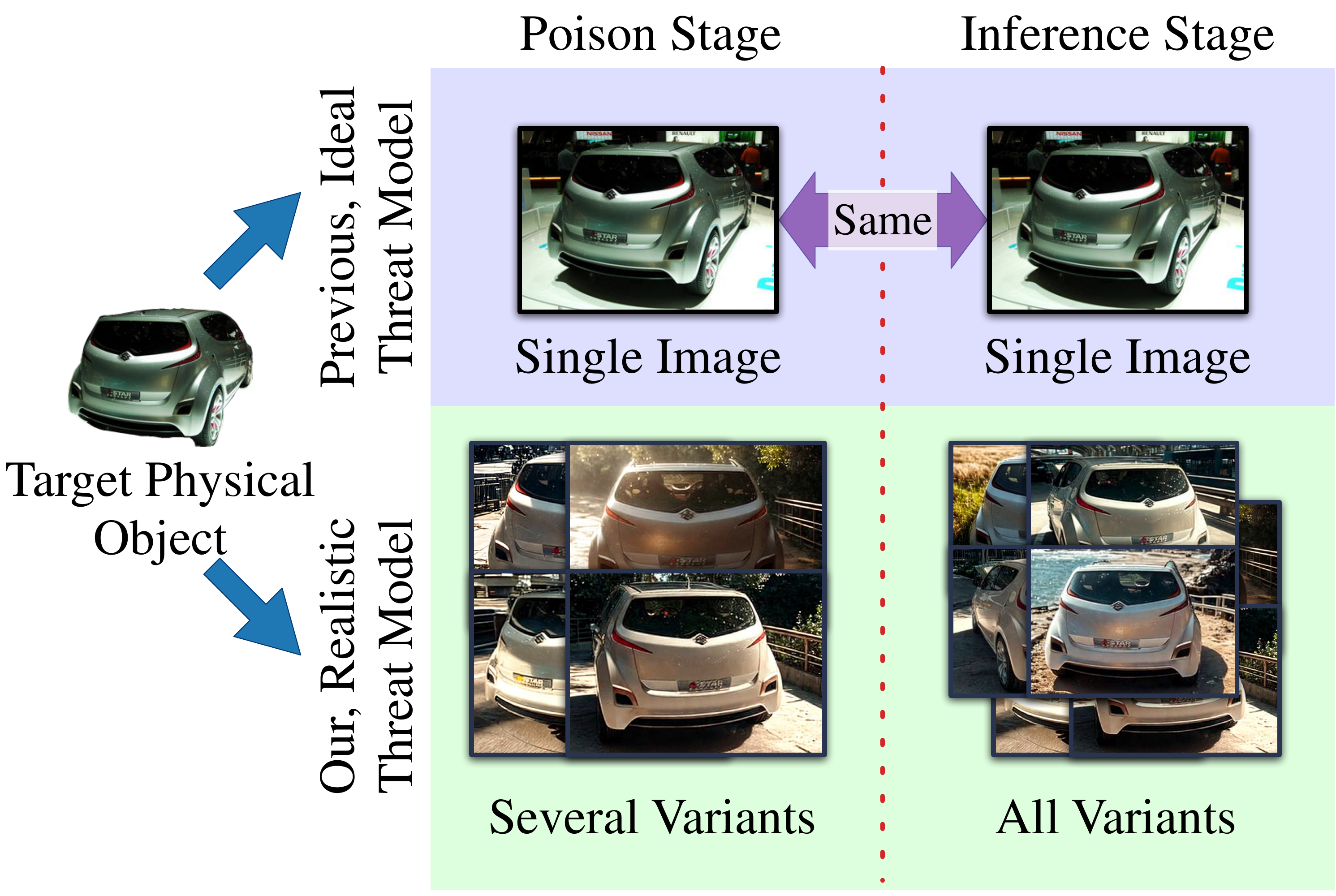}
\caption{
Compared to the previous, overly ideal threat model, ours considers poisoning generalization across realistic target variations regarding viewpoint, background, lighting conditions, etc.
}
\label{fig:threat}
\end{figure}

In this paper, we address the above limitation by conducting the first systematic generalizability study of clean-label target data poisoning considering diverse variations of physical objects.
\Figref{fig:threat} illustrates the key difference between our realistic threat model and the previous one. 
In particular, our threat model is more challenging in two aspects: (i) the poisoner needs to accurately control the poison over a set of images instead of a single image, and (ii) the poisoner is only allowed to utilize a limited number of target object variations during the attack but hoping to generalize to all possible (seen and unseen) variants (Section \ref{sec:threat model}).

Under such a more realistic threat model, we study the state-of-the-art TDP methods that rely on gradient matching between poisoned and target images~\cite{geiping2021witches, souri2022sleeper, yang2022not}.
Our analyses show that solely optimizing the direction (cosine similarity) is insufficient to match the gradients when considering varying physical settings of the target object (Section \ref{sec:motivation}).
To address this limitation, we propose optimizing both the gradient direction and magnitude (Euclidean distance) under the gradient matching framework (Section \ref{sec:detailed design}). Such a novel design enables us to achieve higher poisoning success for \emph{generalizable targeted data poisoning} (G-TDP).

Through comprehensive experiments across multiple image benchmarks, various model architectures, and a range of training paradigms, we demonstrate the superiority of the proposed method against target variations that commonly appear in the physical world, including changes in viewpoint, background, and lighting conditions (Section \ref{sec:exp}).
Specifically for viewpoint, we choose Multi-View Cars \citep{ozuysal2009pose} as the target, where our method outperforms the previous best method~\citep{geiping2021witches} by $19.49\%$ in poisoning success rate while retaining similar validation accuracy, averaged over four models on CIFAR-10.
In addition, we construct a dataset with controllable background and lighting variations, where our method maintains a high success rate of $54.92\%$. In sharp contrast, the rate is $0\%$ without poisoning.
Under the standard TDP settings, our method also attains the best poisoning performance compared with baselines, again confirming the advantages of our attack design (Section \ref{sec:benchmark}).

Below, we summarize our work's main contributions:
\begin{itemize}[leftmargin=0.2in]
    \vspace{0.02in}
    \item We, for the first time, study a realistic threat model of generalizable targeted data poisoning (G-TDP), where physical variations of the target object are considered. 
    \vspace{0.02in}
    \item To address the limitation of gradient matching under our threat model, we propose to leverage both the gradients' direction and magnitude for poison optimization.
    \vspace{0.02in}
    \item Our poisoning method consistently performs the best not only in the realistic threat model but also in the subpopulation settings and the previous threat model.
\end{itemize}

\section{Related Work}
\label{sec:related}

\shortsection{Targeted Data Poisoning (TDP)}
Targeted data poisoning attacks \citep{goldblum2022dataset} tamper with ML model training without modifying test samples during inference.
\citet{koh2017understanding} introduced the initial threat model of TDP (misclassifying specific samples), and \citet{munoz2017towards} formulated it as a bi-level optimization problem. Later,
\citet{shafahi2018poison} proposed clean-label poisoning for specific test images (no label tampering, linear transfer learning), which is further improved by \citet{zhu2019transferable} and \citet{aghakhani2021bullseye}.
\citet{huang2020metapoison} presented MetaPoison, a clean-label TDP method for a training-from-scratch scenario. 
Recently,
\citet{geiping2021witches} introduced Witches' Brew, using gradient matching to enhance success rate and reduce adversarial creation time in training-from-scratch.
Our work follows the aforementioned studies but extends their scope regarding the poisoning target.
Specifically, instead of considering an exact target image, we consider a range of varying physical objects, which is more realistic and challenging.

\shortsection{Physical-Domain Attacks}
Attacks that can be physically realizable have gained increased attention, such as adversarial patches~\cite{LaVAN,Advpatch}, since they are more realistic.
Methods for crafting adversarial patches are either optimization-based~\cite{T-SEA,Advcloak,AdvTshirt,Advpatch}, using gradient descent to perturb localized image areas, or generation-based~\cite{GNAP,DM-NAP,TC-EGA}, which utilizes generative models like GANs and diffusion models.
Natural objects also serve as backdoor triggers~\citep{wenger2021backdoor, wenger2022natural} in backdoor data poisoning.
In comparison, we focus on targeted data poisoning and explore physical object variations.

\section{Generalizable Targeted Data Poisoning}
\label{sec:G-TDP}

\subsection{Preliminaries on TDP}
\label{sec:preliminaries}
In this section, we introduce the necessary notations and definitions to explain the standard threat model of targeted data poisoning.
We work with image classification tasks, and refer to a model with parameter $\theta$ as a function $f_{\theta}:\mathcal{X} \rightarrow \mathcal{Y}$, where $\mathcal{X}\subseteq{\R^d}$ represents the input space of images and $\mathcal{Y}$ denotes the set of possible class labels.
Let $\mathcal{D}$ be the underlying data distribution and $\mathcal{S} = \{(\bm{x}_i, y_i)\}_{i=1}^n$ be a set of training examples, where each individual point $(\bm{x}_i, y)$ is i.i.d. sampled from $\mathcal{D}$. 
Under the supervised learning regime, an image classification model is usually trained by solving the following optimization problem:
\begin{align*}
    \min_\theta \frac{1}{|\mathcal{S}|} \sum_{(\bm{x}, y)\in\mathcal{S}}\ell\left(f_{\theta}(\bm{x}), y\right),
\end{align*}
where $\ell$ stands for the employed training loss function, such as cross-entropy loss, and $|\mathcal{S}| = n$ denotes the set size of $\mathcal{S}$.

Following prior literature \cite{steinhardt2017certified, cina2023wild}, we formalize the problem of targeted data poisoning as a two-party security game between a poisoner and a victim model trainer:
\begin{enumerate}[leftmargin=0.2in]
    \vspace{0.01in}
    \item The poisoner knows a target sample\footnote{Although the original definition of TDP enables the poisoner to control multiple targets, previous clean-label TDP works are often established on the assumption that only one target sample exists~\citep{huang2020metapoison, geiping2021witches}. We follow this typical setting to explain the main idea of our work for simplicity, whereas we test the multi-target extension in our later experiments.}
    $\bm{x}^t$ and the dataset $\gS$ that will be used to train the victim model;
    \vspace{0.02in}
    \item The poisoner's knowledge of the training process and capability depend on specific poisoning scenarios;
    \vspace{0.02in}
    \item The poisoner generates a poisoned training set $\gS_p'$ by poisoning a subset (usually less than $1\%$) of $\gS$;
    \vspace{0.02in}
    \item The victim model trainer employs the poisoned set $\gS_p'$ to train a poisoned model $f_{\theta'}$ with parameters $\theta'$.
\end{enumerate}
\vspace{0.02in}
The poisoner aims to craft a set of data poisons $\mathcal{S}_p'$ such that the model $f_{\theta'}$ trained using the poisoned training set can misclassify many target samples to some intended adversarial class $y_{\mathrm{adv}}$ without sacrificing the \emph{validation accuracy} (i.e., classification accuracy over clean data distribution $\mathcal{D}$).

Aligned with prior work~\citep{aghakhani2021bullseye, geiping2021witches, saha2019hiddentriggerbackdoorattacks, souri2022sleeper}, we consider the setting of clean-label TDP. In particular, the poisoner is allowed to modify at most $k$ samples in $\mathcal{S}$ with imperceptible perturbations bounded in $L_\infty$-norm: $\Delta=\{\bm{\delta}_1, \bm{\delta}_2, \ldots, \bm{\delta}_k\}$, where $\|\bm{\delta}_j\|_\infty \leq \epsilon$ for any $j\in\{1,2,\ldots,k\}$ with $\epsilon$ denoting the $L_\infty$ perturbation strength.
Let $\gS_p = \{(\bm{x}_i^p, y_i^p)\}_{i=1}^k$ denote the subset of clean training samples in $\mathcal{S}$ that are going to be modified by the poisoner. Note that $\gS_p \subseteq \gS$ and $|\gS_p| \leq \alpha \cdot |\gS|$, where $\alpha$ denotes the poisoning budget that is usually a small number like $1\%$. 
After injecting the perturbations, the poisoned training set can be defined as:
\begin{equation}
\label{eq:def poisoned set}
\gS'(\Delta) = (\gS \setminus \gS_p) \cup \gS_p'(\Delta),
\end{equation}
where $\gS \setminus \gS_p$ is the remaining set of unperturbed training data and $\gS_p'(\Delta)$ stands for the set of data poisons:
\begin{equation*}
\gS_p'(\Delta) = \big\{(\bm{x}_i^p + {\bm\delta}_i, y_i^p)\big\}_{i=1}^k.
\end{equation*}
Denote by $\gL$ the averaged training loss with the option of employing some data augmentation techniques:
\begin{equation*}
    \gL(f_\theta, \gS, \mathcal{A}) = \frac{1}{|\gS|} \sum_{(\bm{x}, y) \in \gS} \ell\big(f_\theta\big(\gA(\bm{x})\big), y\big),
\end{equation*}
where $\gA:\mathcal{X}\rightarrow \mathcal{X}$ represents the augmentation operation that the victim may adopt during training. If no data augmentation is used, then $\gA$ is simply the identity mapping.

With all the technical notations in place, the objective of clean-label targeted data poisoning can thus be defined as the following bi-level optimization problem~\citep{munoz2017towards, huang2020metapoison, geiping2021witches}:
\begin{equation}
    \begin{aligned}
        \label{eq:target_bi}
        & \min_{\|\Delta\|_{\infty} \leq \epsilon} \: \ell\left( f_{
        \theta'}(\bm{x}^t), y_{\mathrm{adv}} \right), \\
        & \:\: \text{s.t.} \:\: \theta' = \argmin_\theta \: \gL\big(f_\theta, \gS'(\Delta), \mathcal{A})\big),
    \end{aligned}
\end{equation}
where $\|\Delta\|_\infty$ denotes the elementwise $L_\infty$-norm of $\Delta = \{\bm{\delta}_1,\bm{\delta}_2, \ldots, \bm{\delta}_k\}$, and in our evaluations, we follow the existing work to obtain $\theta'$ using supervised learning. 
In the following discussions, when $\Delta$ is free of context, we write $\gS' = \gS'(\Delta)$ and $\gS_p' = \gS_p'(\Delta)$ for ease of presentation.

\subsection{Our Threat Model: G-TDP}
\label{sec:threat model}

\begin{figure}[!t]
    \includegraphics[width=\linewidth]{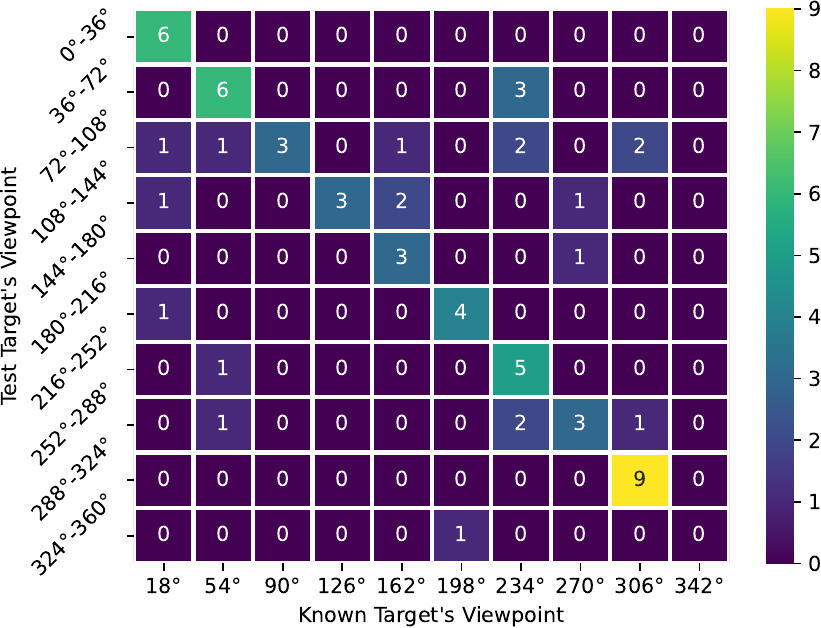}
    \caption{
    Witches' Brew \citep{geiping2021witches} is only effective when the viewpoint known by the poisoner (x-axis) is in the test viewpoint range (y-axis), implying low generalizability across target variations.
    }
    \label{wibHeatmap}
\end{figure}

A critical limitation of the standard threat model of targeted data poisoning is how attack success is measured, as defined in the optimization objective of Equation \ref{eq:target_bi}. Note that the poisoner aims to optimize the data poisons $\mathcal{S}_p'(\Delta)$ and evaluate the attack performance with respect to known target images in $\mathcal{S}_t$.
However, if considering a practical scenario where the poisoned model $f_\theta$ is deployed in the real world, any target image (typically corresponding to an object) may be subject to varying physical conditions, such as viewpoint change, background variations, and lighting alterations. We argue that focusing solely on a given known set of target samples (available to the poisoner) is insufficient to capture the actual objective of the adversary, since it is unclear whether the poisoning success can be transferred to the target object under unknown physical variations. 

Therefore, we introduce a new threat model, termed as \emph{generalizable targeted data poisoning} (G-TDP), which considers possible physical-world variations of target objects.
Formally, given a specific target object $t$ (i.e., a sports car), let $\mathcal{V}_t$ be the distribution capturing the real-world variations of the object (i.e., the same sports car from all the possible viewpoints).
We assume that the poisoner can acquire a small set of target samples from the distribution $\mathcal{V}_t$ to launch the attack (similar to the standard TDP setting). But instead of optimizing for the known targets, the poisoning success rate is evaluated over the whole distribution of target variations $\mathcal{V}_t$ in our threat model. Specifically, G-TDP can be cast into the following bi-level optimization problem:
\begin{equation}
    \begin{aligned}
    \label{eq:G-TDP}
        & \min_{\|\Delta\|_{\infty} \leq \epsilon} \: \mathbb{E}_{\bm{x}^t\sim\mathcal{V}_t} \left[\ell \left( f_{\theta'}(\bm{x}^t), y_{\mathrm{adv}} \right)\right], \\
        & \:\: \text{s.t.} \:\: \theta' = \argmin_\theta \: \gL\big(f_\theta, \gS'(\Delta), \mathcal{A})\big),
    \end{aligned}
\end{equation}
where all the notations are defined in Section \ref{sec:preliminaries}. 

To enable a feasible experimental setting, we first obtain a finite set of target samples $\tilde{\mathcal{S}}_t = \{(\bm{x}_j^t, y_{\mathrm{adv}})\}_{j=1}^M$, where each $\bm{x}_j^t$ is drawn i.i.d. from $\mathcal{V}_t$. Then, $\tilde{\mathcal{S}}_t$ is split into two non-overlapping subsets: $\mathcal{S}_t = \{(\bm{x}_j^t, y_{\mathrm{adv}})\}_{j=1}^m$, which is available to the poisoner to craft data poisons, and the remaining unused variations $\tilde{\mathcal{S}}_t \setminus \mathcal{S}_t$. $M$ is assumed to be much larger than $m$ in our evaluations, representing a more realistic setting where the poisoner is only accessible to a small portion of all possible variations.
In practical, collecting samples incurs costs, we limit the cost here to emphasize the seriousness of the threat.
Note that the key difference between Equation \ref{eq:G-TDP} and Equation \ref{eq:target_bi} is that, under our threat model of generalizable targeted data poisoning, the poisoning success needs to evaluate both available targets ($\mathcal{S}_t$) and unused variations ($\tilde{\mathcal{S}}_t \setminus \mathcal{S}_t$).

\shortsection{Preliminary Experiments} 
To study whether existing TDP methods can be generalized to our threat model, we conduct preliminary experiments to test the performance of a state-of-the-art gradient matching-based TDP method, \emph{Witches' Brew} \citep{geiping2021witches}, on the CIFAR-10 dataset with respect to varying viewpoints of a physical car in the Multi-View Car dataset (see Section \ref{sec:exp} for detailed experimental settings). 
Figure \ref{wibHeatmap} shows the heatmap regarding the number of successful poisons within the test range of angles (y-axis), when crafting data poisons using Witches' Brew based on a known target's viewpoint (x-axis). Specifically, we set $m=1$ in this experiment, meaning that each entry in the heatmap corresponds to a single given target's viewpoint, whereas during inference, we test against all the viewpoints of the target ($9$ per row). 
Our preliminary results suggest that the prior TDP method does not generalize well to unseen viewpoint variations, reflected by low values in the off-diagonal entries in Figure \ref{wibHeatmap}. The poisoning success rate is higher only when the employed target view for optimizing the data poisons falls within the test viewpoint range of the target. 

\section{Our Poisoning Method}
\label{sec:method}

\subsection{Motivation} 
\label{sec:motivation}

The gradient matching technique proposed in \citep{geiping2021witches} is based on the insight that if the poisoner can create poison samples such that for any $\theta$ encountered during the training process, the following two gradients are similar:
\begin{equation}
    \label{eq:gm}
    \frac{1}{k} \sum_{i=1}^k \nabla_\theta \ell(f_\theta(\bm{x}^p_i + \bm\delta_i), y_{\mathrm{adv}}) \approx \frac{1}{m} \sum_{j=1}^m \nabla_\theta \ell(f_\theta(\bm{x}_j^t), y_{\mathrm{adv}}),
\end{equation}
where $m$ is the number of available target samples and $k$ is the number of data poisons, the model behavior on the poison samples may perform similarly to that on the target samples.
However, since the model gradients will dynamically change during the training process depending on the data poisons, there is no guarantee that such a condition can be satisfied. 
As an alternative, the authors proposed matching the gradients' direction in cosine distance under an alternating minimization framework as a heuristic to approximately achieve the condition specified by Equation \ref{eq:gm}.
To be more specific, the gradient matching objective is defined based on the following cosine similarity loss:
\begin{equation}
    \label{eq:cos}
    D_{\mathrm{cos}}(\gL_p, \gL_t) = 1 - \cos(\nabla_\theta \gL_p, \nabla_\theta \gL_t),
\end{equation}
where $\mathcal{L}_p$ and $\mathcal{L}_t$ represent the averaged model loss over the set of data poisons and the target samples, respectively:
\begin{align}
    \label{eq:aug}
    \gL_p = \gL(f_\theta, \gS_p', \gA), \:\: \gL_t = \gL(f_\theta, \gS_t, \varnothing).
\end{align}
Note that we disable the option of data augmentation in $\mathcal{L}_t$, because the test samples will not be augmented during inference.
Looking at Equation \ref{eq:aug}, one may think that performing data augmentation also on the given samples in $\mathcal{S}_t$ for improving generalization across unseen target variations.
However, our exploratory experiments show that it is not the case (see Appendix~\ref{app:augment}). This suggests that digital transformations are fundamentally different from physical (object) variations.

Although relying on cosine distance to realize gradient matching is sensible, it is not always optimal. For small or thin victim models (e.g., a ResNet18 with smaller width in convolutional layers), directly optimizing squared Euclidean distance is more effective than optimizing $D_{\mathrm{cos}}$ (see the remark in Section 3.3 of \citep{geiping2021witches} for discussions).
The adversarial loss with Euclidean distance is given by:
\begin{equation}
     \label{eq:ed}
     D_{\mathrm{ED}}(\gL_p, \gL_t) = \| \nabla_\theta \gL_p - \nabla_\theta \gL_t\|_2^2.
\end{equation}

\noindent\textbf{Why the Magnitude of Gradients Matters?}
$D_{\mathrm{ED}}$ measures the magnitude difference between the model gradients of data poisons and the known target samples.  
Gradients with larger magnitudes exert a more significant influence on the model's parameters during gradient descent. 
Considering our TDP problem, such gradients can induce a more substantial drop in the target samples' loss, and this drop can even generalize to semantically similar samples (e.g., other views of a physical car). For more details, see \textbf{Proposition 1} and its proofs in Appendix~\ref{app:augment}. 
However, the clean versions of our poison samples are correctly labeled, meaning they inherently possess gradients with small magnitudes.
Suppose these clean-label poison samples are optimized without constraining their gradient magnitudes. In that case, the magnitude of their gradients is not guaranteed to be sufficiently large to counteract the influence from clean samples.

The magnitude of the target gradient is large due to their incorrect labeling (labeled with $y_\mathrm{adv}$). 
Optimizing the Euclidean distance between the gradients of poison samples and those of target samples increases the magnitude of the poison gradient. 
Given that the Euclidean distance is non-negative and satisfies the definition of gradient-matching poisoning, it is ideal for incorporation into the adversarial loss. 
When the gradients of data poisons exhibit the same cosine similarity to those of target samples as when directly optimizing $D_{\mathrm{cos}}$ but with a smaller Euclidean distance, they can yield better performance and generalization.

\begin{figure}[t]
    \centering
    \subfloat[Cosine Similarity]{
        \includegraphics[width=0.48\linewidth,height=0.40\linewidth]{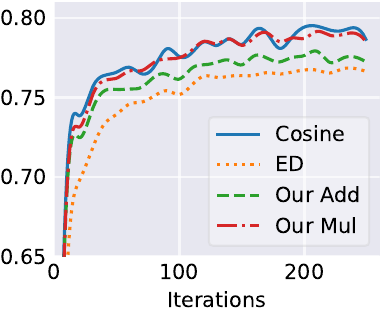}
    }
    \subfloat[Euclidean Distance]{
        \includegraphics[width=0.48\linewidth,height=0.40\linewidth]{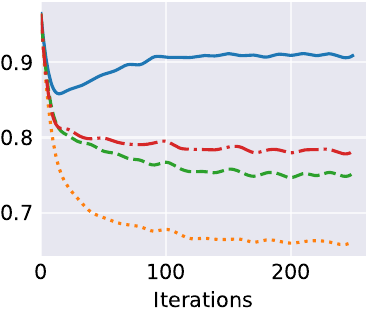}
    }
    \caption{(a) Cosine Similarity and (b) Euclidean distance curves during the optimization of gradient matching with various adversarial losses. Here, the Euclidean distance is regularized.}
    \label{fig:single_aug}
\end{figure}

\begin{algorithm*}[!t]
    \caption{Clean-label Generalizable Targeted Data Poisoning}
    \label{alg:sp_attack}
    \KwIn{Victim model $f_\theta$, known variations of the target object $\gS_t$, adversarial intended class $y_{\mathrm{adv}}$, clean training set $\gS$, perturbation strength $\epsilon$, optimization steps $S$, retraining times $R$}
    
    Randomly initialize perturbations $\Delta = \{\bm{\delta}_1, \ldots, \bm{\delta}_k\}$ with $\|\Delta\|_\infty \leq \epsilon$;

    $\gS_p' \leftarrow \{(\bm{x}_i^p + \bm{\delta}_i, y_{\mathrm{adv}})\}_{i=1}^k$; \algorithmiccomment{\textit{$\gS_p$ will dynamically change with $\Delta$}}

    \For{$s=1, 2, \cdots, S$}{
        Compute the adversarial loss $D_{\mathrm{mul}}(\gL_p, \gL_t)$ with respect to $\gS_p'$ and $\gS_t$;
        \algorithmiccomment{\textit{$D_{\mathrm{mul}}$ is defined in Equation \ref{eq:mul}}}
    
        Update $\Delta$ using projected gradient descent on $D_{\mathrm{mul}}(\gL_p, \gL_t)$;
        \algorithmiccomment{$D_{\mathrm{mul}}$ can be replaced by other losses}

        Project $\Delta$ onto $\| \Delta \|_\infty \leq \epsilon$;

        Retrain the victim model using poisoned set $\gS'$ every $\lfloor S / (R + 1)\rfloor$ iterations except $s=S$;
    }

    \KwOut{Poisoned training dataset $\gS'$}
\end{algorithm*}

\subsection{Detailed Design}
\label{sec:detailed design}
Based on the above finding, we need to optimize the gradient direction and Euclidean distance simultaneously.
Therefore, we introduce two possible loss functions to combine the $D_{\mathrm{ED}}$ and $D_{\mathrm{cos}}$ distances:
\begin{align}
    \label{eq:add}
    D_{\mathrm{add}}(\gL_p, \gL_t) &= \frac{\sqrt{D_{\mathrm{ED}}(\gL_p, \gL_t)}}{\| \nabla_\theta\gL_t\|_2} + D_{\mathrm{cos}}(\gL_p, \gL_t), \\
    \label{eq:mul}
    D_{\mathrm{mul}}(\gL_p, \gL_t) &= D_{\mathrm{ED}}(\gL_p, \gL_t) \cdot D_{\mathrm{cos}}(\gL_p, \gL_t).
\end{align}
In particular, because the $D_{\mathrm{ED}}$ and $D_{\mathrm{cos}}$ losses are not at the same order of magnitude (for example, $D_{\mathrm{ED}}$ is usually more than $10,000$ on ResNet-18), Equation \ref{eq:add} regularizes the $D_{\mathrm{ED}}$ with the $L_2$-norm of the model gradients on the target samples.
Equation \ref{eq:mul} further alleviates the problem by multiplying $D_{\mathrm{ED}}$ and $D_{\mathrm{cos}}$.
Figure~\ref{fig:single_aug} validates that the above designs meet our expectations.
More specifically, $D_{\mathrm{mul}}$ matches the gradients' direction better, while $D_{\mathrm{add}}$ is better at minimizing the Euclidean distance.

\Algref{alg:sp_attack} shows the pseudocode of our proposed poisoning method.
We also incorporate a retraining step~\citep{souri2022sleeper} to dynamically capture the effect of current perturbations (see ablation studies in \Secref{sec:ablation} for supporting evidence).

\begin{table*}[t]
\centering
\caption{Poisoning results with target physical objects from the Multi-View Car dataset. Cosine and ED are state-of-the-art TDP losses from \cite{geiping2021witches}. ``$\mathrm{ImageNet}^*$'' denotes the commonly adopted subset of ImageNet with 100 classes~\cite{huang2021unlearnable, fu2022robust, wang2024provably}. We also evaluate the effectiveness on the complete ImageNet, and the poisoning success rate is $99.21\%$. The victim model and the poisoner's surrogate model are the same.}
\label{tab:cifar10-whitebox}
\resizebox{\textwidth}{!}{
\begin{tabular}{ll|ccccc|ccccc}
\toprule
\multirow{2.4}{*}{\textbf{Dataset}}       & \multirow{2.4}{*}{\textbf{Victim Model}} & \multicolumn{5}{c|}{\textbf{Validation Accuracy} (\%)} & \multicolumn{5}{c}{\textbf{Poisoning Success Rate} (\%)} \\ \cmidrule(l){3-12} 
&                               & w/o & Cosine    & ED & $\mathrm{Our}_{\mathrm{add}}$     & $\mathrm{Our}_{\mathrm{mul}}$    & w/o    & Cosine & ED   & $\mathrm{Our}_{\mathrm{add}}$ & $\mathrm{Our}_{\mathrm{mul}}$           \\ \midrule
\multirow{4}{*}{CIFAR-10} & ConvNet64 & $85.21$ & $85.09$ & $84.62$ & $84.87$ & $84.91$ & $0.00$ & $30.38$ & $75.43$ & $\mathbf{82.30}$ & $76.48$ \\
                          & VGG11 & $88.89$ & $88.89$ & $88.66$ & $88.75$ & $88.80$ & $0.67$ & $75.15$ & $77.30$ & $80.18$ & $\mathbf{84.44}$ \\
                          & ResNet-18 & $92.15$ & $91.99$ & $91.94$ & $91.98$ & $91.99$ & $1.11$ & $82.40$ & $87.29$ & $88.85$ & $\mathbf{90.13}$ \\
                          & MobileNet-V2 & $88.21$ & $88.10$ & $87.74$ & $87.82$ & $87.93$ & $0.64$ & $67.03$ & $68.09$ & $73.69$ & $\mathbf{81.87}$ \\ \midrule
$\mathrm{ImageNet}^*$                  & ResNet-18 & $75.03$ & $74.82$ & $74.77$ & $74.84$ & $74.79$ & $0.00$ & $48.41$ & $16.53$ & $45.13$ & $\mathbf{56.16}$ \\ \bottomrule
\end{tabular}
}
\end{table*}

\section{Experiments}
\label{sec:exp}
We conduct experiments to test the efficacy of our design on image benchmarks and consider physical targets with viewpoint variations derived from the Multi-View Car dataset. 
More specifically, we compare our G-TDP losses with previous TDP losses under both standard (the poisoner knows the victim model's architecture) and transfer settings (Section \ref{sec:exp results}). 
Further, we explore the robustness of our method against poisoning defenses (Section \ref{sec:defense}).
To examine other types of variations, we create a handmade dataset that involves viewpoint, background, and lighting changes, and test our method's performance (Section \ref{sec:eval_misc}) \footnote{The implementations of our method are available at: \url{https://github.com/zhizhen-chen/generalizable_tdp}.}.

\subsection{Experimental Setup}
\label{sec:end2end}

\begin{figure}[t]
    \centering
    \subfloat[Multi-View Car]{
        \label{fig:visualization multiview}
        \includegraphics[width=0.48\linewidth,height=0.30\linewidth]{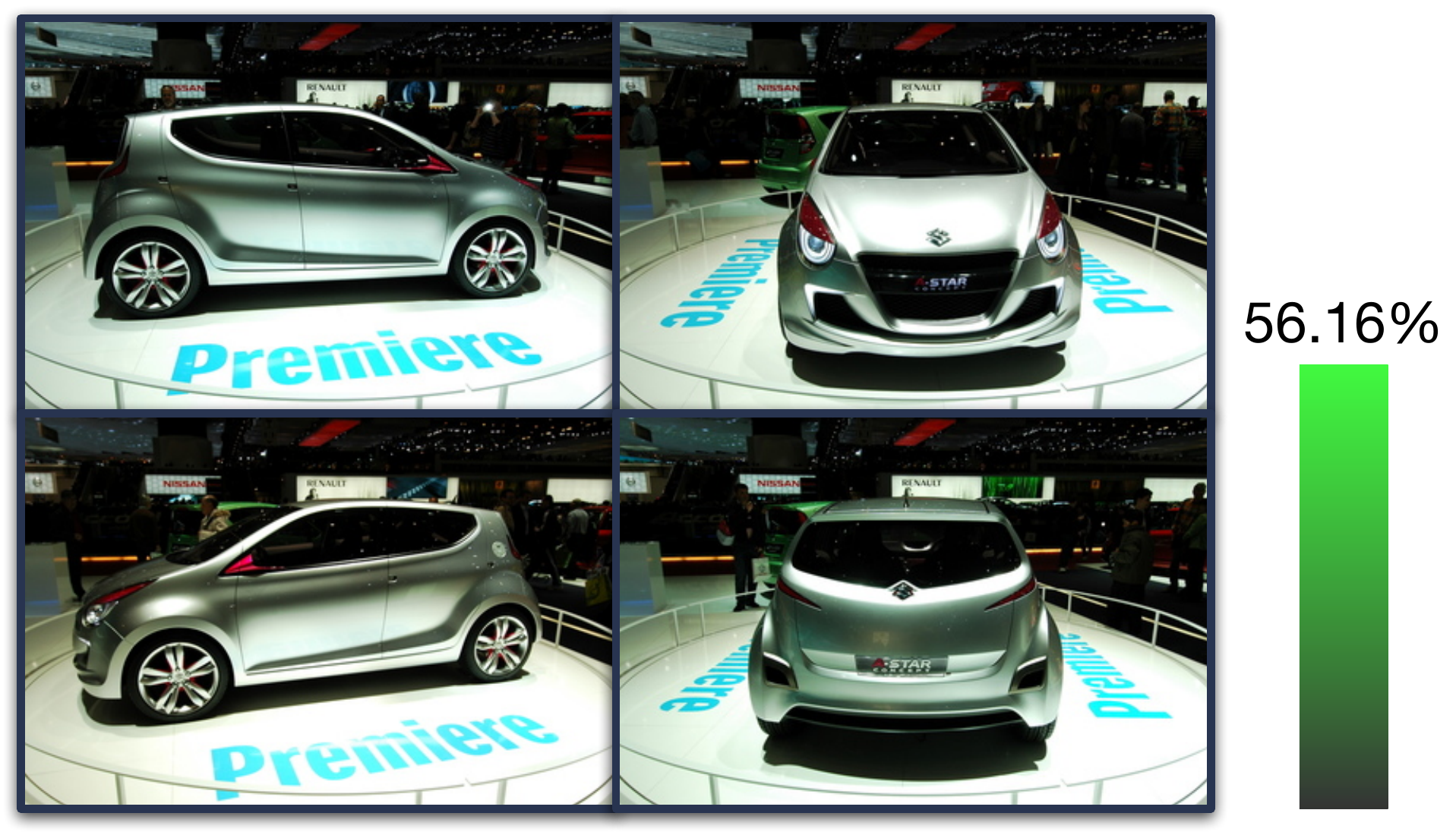}
    }
    \subfloat[Handmade]{
        \label{fig:visualization handmade}
        \includegraphics[width=0.48\linewidth,height=0.30\linewidth]{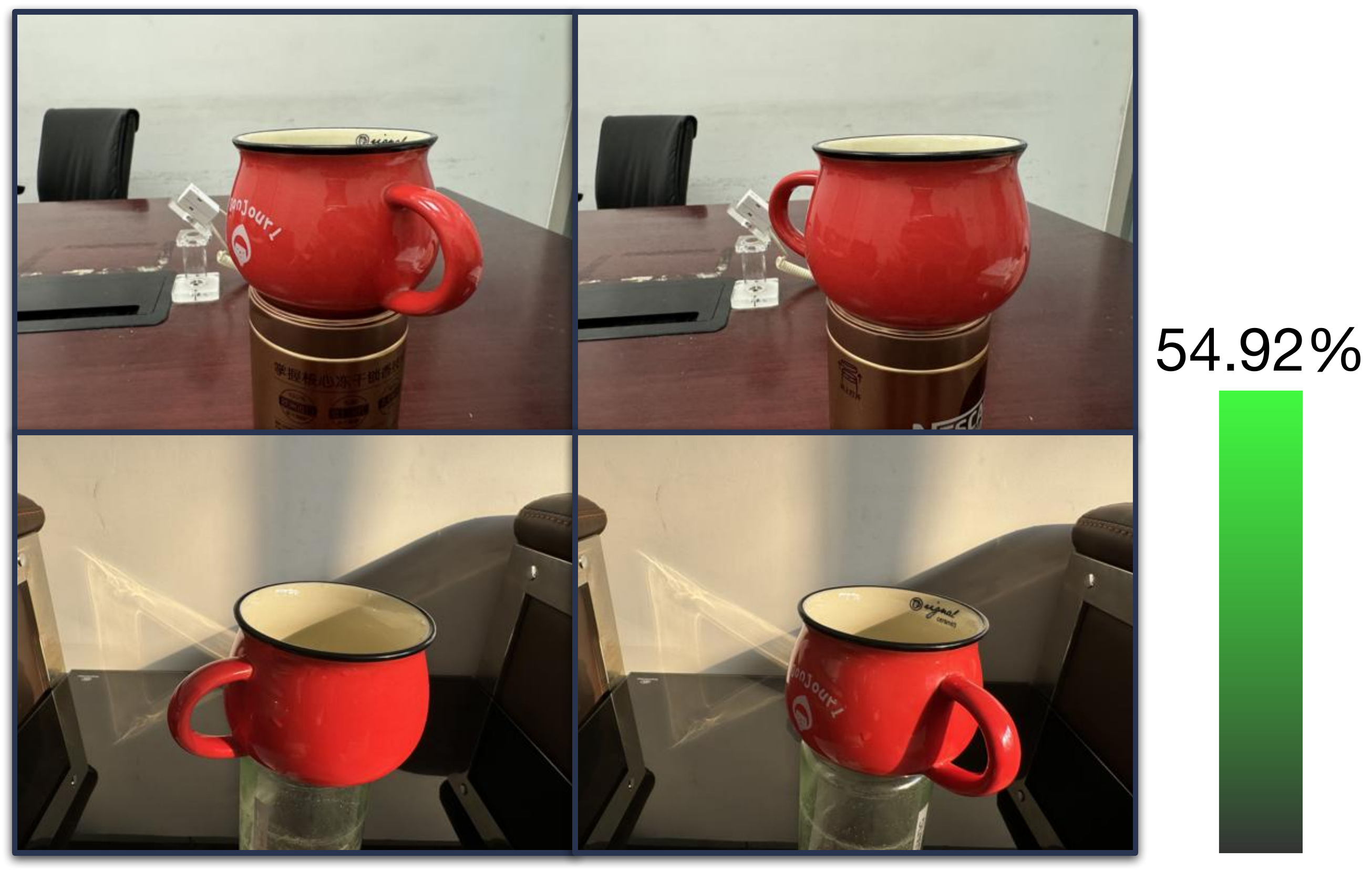}
    }
    \caption{Visualizations of target samples in the Multi-View Car and the handmade datasets. The green bar beside the examples denotes the poisoning success rate on ImageNet.}
    \label{fig:visualization}
\end{figure}

\shortsection{Target Object Variations}
The target physical objects are derived from the Multi-View Car dataset \citep{ozuysal2009pose}, which provides $20$ distinct cars, each with around $100$ photos taken with a $360$-degree rotation.
We regard a specific car as a target physical object $\bm{x}^t$ and the viewpoint changes as the variations $\gV_t$.
This setting is ideal for conducting experiments in our threat model because it contains challenging viewpoint variations and moderate background variations.
We provide the visualizations of target samples in Figure~\ref{fig:visualization multiview}.

As mentioned in Section~\ref{sec:threat model}, the difficulty of G-TDP against physical objects is to enhance the generalizability of the poisoning with a limited number of known target images. 
To examine this capability, we follow \citet{aghakhani2021bullseye}, providing the poisoner with $10$ images taken from different angles while the remaining images are designated test samples.
We rigorously separate the samples for the poisoner and for validation to prevent data leakage, just like what we do in other machine learning problems.
This measure offers an additional benefit: it accounts for images that are unknown to the poisoner, as the poisoner cannot always anticipate the target image during the inference time.

\shortsection{Datasets \& Configurations}
We evaluate our method under G-TDP with target physical objects on CIFAR-10 and a subset of ImageNet.
The poisoner contaminates $1\%$ of the CIFAR-10 training set with the following hyperparameters used in Algorithm \ref{alg:sp_attack}: $\epsilon=16$, $S=250$, and $R=4$. 
Both the poisoner and the victim enable the differentiable data augmentation during the model training process~\cite {geiping2021witches}.
The subset of ImageNet contains $100$ classes, which is a typical setting adopted by prior work \citep{huang2021unlearnable, fu2022robust, wang2024provably}.
We make this reduction for the sake of efficiency.
In particular, we select the first $100$ classes and replace the last few classes with those we are interested in, including ``sports car'' and a few other classes discussed in Section~\ref{sec:eval_misc}.
The poisoner contaminates $0.5\%$ of the subset of the ImageNet. 
The hyperparameters are $\epsilon=8$, $S=250$, $R=2$.
Besides, we conduct a single trial over the complete ImageNet dataset to validate whether the poisoning is effective in the large-scale training.

\begin{table}[!t]
\centering
\caption{SR (\%) evaluated under transferability settings. The poisoner constructs poisoned samples via a VGG11 surrogate model.}
\label{tab:transfer}
\resizebox{0.95\linewidth}{!}{
\begin{tabular}{l|cccc}
\toprule
\textbf{Victim} & Cosine & ED & $\mathrm{Our}_{\mathrm{add}}$ & $\mathrm{Our}_{\mathrm{mul}}$ \\
\midrule
ResNet-18    & $61.13$ & $69.67$ & $75.39$ & $\mathbf{80.17}$ \\
MobileNet-V2 & $41.46$ & $74.20$ & $\mathbf{74.89}$ & $74.42$ \\
LeViT-384    & $51.34$ & $59.45$ & $60.31$ & $\mathbf{61.09}$ \\
Average      & $51.31$ & $67.77$ & $70.20$ & $\mathbf{71.89}$ \\
\bottomrule
\end{tabular}
}
\end{table}

If not explicitly mentioned, the default setting of experiments is conducted based on the ResNet-18 using the CIFAR-10 dataset.
We repeat each experiment using $10$ different seeds to mitigate the impact of randomness on the results, and the poison samples corresponding to each seed are evaluated through $8$ model training trials. More experimental details and visualization figures are provided in Appendix~\ref{app:experiment details} and \ref{app:visualization}.

\subsection{Main Results}
\label{sec:exp results}

Table~\ref{tab:cifar10-whitebox} demonstrates the \emph{poisoning success rate} (SR) in various model architectures toward varying physical objects.
The experiment setting is standard, where the victim model is the same as the surrogate model used by the poisoner, just as in the previous works \cite{huang2020metapoison, geiping2021witches}.
As a result, our proposed framework and improved adversarial losses can significantly enhance the targeted poisoning's generalizability in various settings.
$D_{\mathrm{ED}}$ consistently achieves a high SR when the training dataset is CIFAR-10, while its SR has plunged when the training dataset is ImageNet.
The same phenomenon also occurs with $D_{\mathrm{add}}$, where the Euclidean distance accounts for a larger proportion than $D_{\mathrm{mul}}$.
$D_{\mathrm{mul}}$ achieves the best poisoning results in various settings, except when the victim model is ConvNet64. 
Therefore, $D_{\mathrm{mul}}$ is a good choice for the adversarial loss for G-TDP.

We also present the side effects on the validation accuracy of the validation set.
$D_{\mathrm{cos}}$ achieves the minimum impact on the validation accuracy.
While $D_{\mathrm{mul}}$ results in a high SR, its impact on the validation accuracy is smaller than $D_{\mathrm{ED}}$ and $D_{\mathrm{add}}$.

\begin{table}[t]
\centering
\caption{SR (\%) of gradient matching when \emph{Victim} (V) and/or \emph{Poisoner} (P) apply or not apply data augmentations.}
\label{tab:augment}
\resizebox{\linewidth}{!}{
\begin{tabular}{l|cc|cccc}
\toprule
\textbf{Dataset} & \textbf{V} & \textbf{P} & Cosine & ED & $\mathrm{Our}_{\mathrm{add}}$ & $\mathrm{Our}_{\mathrm{mul}}$ \\ \midrule
\multirow{4}{*}{CIFAR-10} & \checkmark    & \checkmark      & $82.40$ & $87.29$ & $88.85$ & $\mathbf{90.13}$ \\
& {\sffamily X} & \checkmark      & $12.14$ & $29.13$ & $\mathbf{34.97}$ & $33.57$ \\
& \checkmark    & {\sffamily X}   & $38.76$ & $\mathbf{66.09}$ & $63.51$ & $62.88$ \\ 
& {\sffamily X} & {\sffamily X}   & $39.26$ & $81.17$ & $\mathbf{81.31}$ & $79.42$ \\ \midrule \midrule
\multirow{2}{*}{$\mathrm{ImageNet}^\mathrm{*}$} & \checkmark    & \checkmark      & $48.41$ & $16.53$& $45.13$ & $\mathbf{56.16}$ \\
& {\sffamily X} & {\sffamily X}   & -       & -       & $25.94$ & $\mathbf{26.54}$ \\
\bottomrule
\end{tabular}
}
\end{table}

\shortsection{Cross-Model Transferability} 
We study the effectiveness of our proposed poisoning method in a more challenging scenario, where the poisoner does not know the model architecture adopted by the victim. 
The poisoner creates the poison samples via a surrogate model, and we evaluate them in another victim model with a different architecture.
Here, we adopt VGG11 as the surrogate model due to its superior transferability (see Section B.2 in our supplemental materials for details).
Table~\ref{tab:transfer} summarizes the results, where $D_{\mathrm{mul}}$ achieves the best average SR, and
$D_{\mathrm{add}}$ shows similar transferability to $D_{\mathrm{mul}}$. This validates that our proposed adversarial losses have better performance than previous ones.
Besides, although the poison samples are constructed in a CNN-based model (VGG11), they can successfully transfer to the LeViT-384 \citep{graham2021levit} model, which involves attention layers. 
We confirm this result in a larger modern ViT model, Swin V2, which achieves a success rate of $41.03\%$.

\shortsection{Data Augmentation for Poisoner and Victim}
We find that TDP is sensitive to data augmentation.
Table~\ref{tab:augment} shows the experimental results of altering the augmentation setting, where the SR has dropped significantly when the data augmentation settings of the victim and poisoner are not aligned.
This result is consistent with the observations in \citep{geiping2021witches}.
$D_{\mathrm{mul}}$ achieves lower performance than $D_{\mathrm{ED}}$ and $D_{\mathrm{add}}$ in some cases when the training dataset is CIFAR-10, but it consistently performs better when the training dataset is ImageNet.
We also find that when the poisoner does not know the victim's augmentation setting, a decent impact can be consistently achieved by disabling data augmentation.

\subsection{Defenses against Data Poisoning}
\label{sec:defense}

\begin{figure}[t]
    \centering
    \subfloat[DP~\cite{hong2020effectiveness}]{              \includegraphics[width=0.50\linewidth,height=0.36\linewidth]{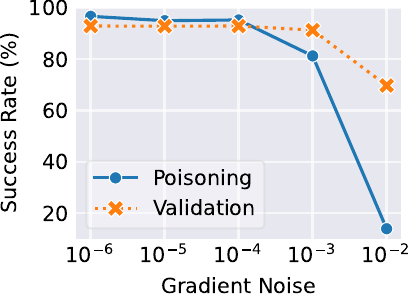}
    }
    \hspace{0.05in}
    \subfloat[EPIc~\cite{yang2022not}]{            \includegraphics[width=0.42\linewidth,height=0.35\linewidth]{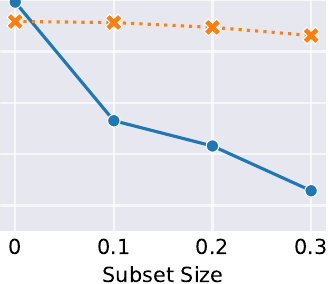}
    }
    \caption{Poisoning success rates against robust training defenses.}
    \label{fig:robust}
\end{figure}

\shortsection{Robust Training}
Robust training refers to a defense that alters the model training process to avoid the influence of the poison samples, but does not modify the samples in the training set.
In particular, we test two robust training methods: \textit{Differential Privacy} (DP) \citep{hong2020effectiveness} and \textit{Effective Poison Identification} (EPIc) \citep{yang2022not}.
Figure~\ref{fig:robust} shows the effectiveness of the two methods, where DP can defend against our poisoning method at the cost of sacrificing much of the model's validation accuracy. 
Although EPIc incurs a smaller cost in terms of accuracy when conducting effective defense, $20\%$ of images of the physical target will still be affected.

\shortsection{Poison Removal}
Poison removal refers to a defense that transforms all training samples in the training set to detoxify the training set, which does not change the model training process.
We test three poison removal methods discussed in \citet{liu2023image}: Gaussian blur, JPEG compression, and bit depth reduction (BDR).
These methods are originally designed to mitigate indiscriminate data poisoning, and we expect these defenses to also be effective against TDP.
Table~\ref{tab:defenses} reports the different defense strategies' performance, with a tradeoff between the defense performance and the model's classification accuracy.
Specifically, JPEG and BDR invalidated the poison samples, whereas BDR had a smaller impact on the validation accuracy.

\begin{table}[!t]
\centering
\caption{Validation accuracy ($\%$) and SR ($\%$) achieved by our method against different poison removal defenses.}
\label{tab:defenses}
\resizebox{0.95\linewidth}{!}{
\begin{tabular}{l|c|ccc}
\toprule
\textbf{Defense}           & w/o     & Gaussian & JPEG    & BDR \\ \midrule
Validation Acc. & $91.99$ & $91.92$  & $83.60$ & $89.28$ \\ 
Success Rate        & $90.13$ & $90.02$  & $0.07$  & $2.48$  \\
\bottomrule
\end{tabular}
}
\end{table}

\subsection{Poisoning to Other Physical Targets}
\label{sec:eval_misc}

Although we conduct comprehensive experiments on poisoning the physical cars in the Multi-View Car dataset, only the viewpoint and background variations are considered.
To find out whether our proposed method is effective on other physical objects and other variations, including viewpoint, background, and lighting variations, we conducted additional experiments on the target physical objects derived from our handmade dataset.
All physical objects in this dataset are photographed from different viewpoints in two distinct scenes (see Figure~\ref{fig:visualization handmade} for illustrations).
We follow the previous setting on ImageNet, while we designated five other targets. 
Due to the randomness of photography, the number of images of each object is different. 
We provide $10$ images in each scene (total $20$ images) to the poisoner.
The remaining images are used to evaluate the attack's performance.
As a result, our proposed method achieved $54.92\%$ SR.
We also compose a single trial on the complete ImageNet dataset, and the SR is $100.0\%$.

\section{Further Analysis}
\label{sec:discussion}

\subsection{Evaluation under Previous Threat Model}
\label{sec:benchmark}

\begin{table}[!t]
\centering
\caption{SR (\%) under the previous TDP threat model \citep{schwarzschild2021just}, where the poisoner uses ResNet-18 as the surrogate model. Here, $K$ is the number of assembled models with different initializations.}
\label{tab:benchmark}
\resizebox{\linewidth}{!}{
\begin{tabular}{l|cccc}
\toprule
\textbf{Method}        & RN-18 & MN-V2 & VGG11 & Avg. \\ \midrule
BP \citep{aghakhani2021bullseye} & $3.00$ & $3.00$ & $1.00$ & $2.33$ \\
WiB \citep{geiping2021witches} ($K=1$) & $45.00$ & $36.00$ & $8.00$ & $29.67$ \\
WiB ($K=4$)      & $55.00$ & $37.00$ & $7.00$ & $33.00$ \\ \midrule
ED ($K=1$) & $19.00$ & $7.00$ & $\mathbf{30.00}$ & $18.67$ \\
$\mathrm{Our}_{\mathrm{mul}}$ ($K=1$, Retrain)  & $38.00$ & $27.00$ & $6.00$ & $23.67$ \\
$\mathrm{Our}_{\mathrm{mul}}$ ($K=1$)           & $52.00$ & $42.00$ & $20.00$ & $38.00$ \\
$\mathrm{Our}_{\mathrm{mul}}$ ($K=4$)           & $\mathbf{61.00}$ & $\mathbf{57.00}$ & $19.00$ & $\mathbf{45.67}$ \\  \bottomrule
\end{tabular}
}
\end{table}

According to Section~\ref{sec:method}, we expect our loss design can also achieve poisoning performance under previous TDP settings.
We replace the adversarial loss in the gradient matching framework~\citep{geiping2021witches} and validate its performance in the benchmark proposed by \citet{schwarzschild2021just}.
Table~\ref{tab:benchmark} demonstrates that our loss $D_{\mathrm{mul}}$ achieves the best performance both in SR and transferability to different model architectures.
Although $D_{\mathrm{ED}}$ achieves a better performance than $D_{\mathrm{cos}}$ in our threat model, its performance is significantly lower than $D_{\mathrm{cos}}$ in other settings.
Besides, adding a retraining step does not help boost TDP's performance in this scenario, which is worth further study.

\subsection{Subpopulation Data Poisoning}
\label{sec:subp}

\begin{table}[t]
\centering
\caption{SR (\%) under subpopulation data poisoning settings.}
\label{tab:subpopulation}
\resizebox{0.95\linewidth}{!}{
\begin{tabular}{l|cccc}
\toprule
\textbf{Victim Model} 
             & Cosine   & ED       & $\mathrm{Our}_{\mathrm{add}}$ & $\mathrm{Our}_{\mathrm{mul}}$ \\ \midrule
ConvNet      & $14.47$ & $48.46$  & $49.85$  & $\mathbf{50.70}$ \\
VGG11        & $47.45$ & $58.27$  & $56.10$  & $\mathbf{61.72}$ \\
ResNet-18    & $59.94$ & $65.86$  & $69.10$  & $\mathbf{71.24}$ \\
MobileNet-V2 & $37.47$ & $64.20$  & $58.35$  & $\mathbf{65.26}$ \\ \bottomrule
\end{tabular}
}
\end{table}

Our method's superior generalization performance makes it a promising solution for threat models other than TDP.
Subpopulation data poisoning, proposed by \citet{jagielski2021subpopulation}, focuses on manipulating the model's behavior over a set of test target samples. Unlike our focus on a specific physical object with varying conditions, these test samples usually correspond to multiple objects.
Nevertheless, our method can be adapted to conduct clean-label subpopulation data poisoning when the subpopulation images share similar visual features (e.g., from the same label class).

We follow the experiment setting with the CIFAR-10 dataset described in Section~\ref{sec:end2end} but alter the target object from a car to a bird species, where the target images are from the CUB-200-2011 dataset \citep{wah2011caltech}.
Note that there are more variations in a bird species than in a multi-view car.
The poisoner has access to $20$ samples from the target bird species, and the poisoning success rate is computed using the remaining samples from the subpopulation.
Table~\ref{tab:subpopulation} reports the SR with respect to the subpopulation target, where the adversarial loss $D_{\mathrm{mul}}$ keeps the best performance, suggesting the potential of our method for the more challenging task of clean-label subpopulation data poisoning.

\subsection{Ablation Studies}
\label{sec:ablation}

\shortsection{Poisoning budget}
Increasing the poisoning budget (i.e., the number of crafted poison samples) can incur a higher probability for the poisoner to be spotted~\cite{carlini2024poisoning}.
Therefore, the threat is more serious if the poisoner achieves the objective at a lower poisoning budget.
Figure~\ref{fig:ablation budget} depicts the relationship between the poisoning budget and the success rate.
A higher budget brings a higher SR, and our method can achieve decent performance even if the poisoning budget is $0.2\%$ of the training set size.
The loose target setting (subpopulation) is more sensitive to the poisoning budget.

\shortsection{Known Target Variants}
G-TDP is expected to be sensitive to the number of available target variants since it indicates how complete the poisoner's knowledge of the target physical object is.
Figure~\ref{fig:ablation variation} shows that SR achieved by our method can be improved by increasing the number of known target variants. 
However, a large number also suggests a higher cost of collecting the variants.

\begin{figure}[!t]
    \centering   
    \subfloat[Poisoning Budget]{
        \label{fig:ablation budget}
        \includegraphics[width=0.48\linewidth,height=0.315\linewidth]{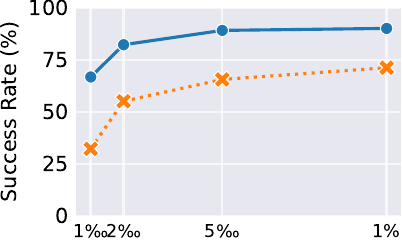}
    }
    \subfloat[Known Target Variants]{
        \label{fig:ablation variation}
        \includegraphics[width=0.44\linewidth,height=0.305\linewidth]{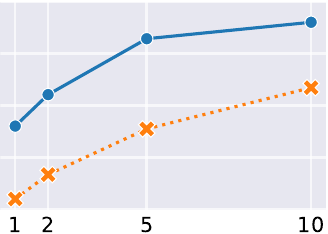}
    }
    \newline
    \subfloat[Perturbation Size]{
        \includegraphics[width=0.48\linewidth,height=0.315\linewidth]{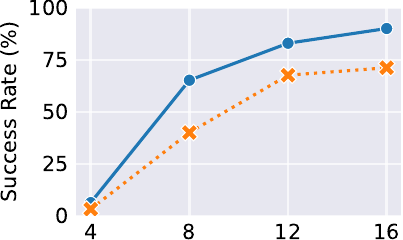}
        \label{fig:ablation size}
    }
    \subfloat[Retraining Times]{
        \includegraphics[width=0.44\linewidth,height=0.305\linewidth]{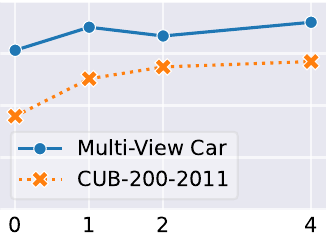}
        \label{fig:ablation retraining}
    }
    \caption{Ablations on Multi-View Car and CUB-200-2011.}
    \label{fig:ablation}
\end{figure}

\shortsection{Perturbation Size}
To satisfy the clean-label requirement, a $L_\infty$-norm $\epsilon$-bound is used to limit the perturbation size.
A smaller $\epsilon$ will bring a more subtle alteration to the poison samples.
Figure~\ref{fig:ablation size} reports how different $\epsilon$ impacts the SR, where the $\epsilon$ influences it more significantly than the poisoning budget.
Our poisoning method achieves a decent performance only for the $\epsilon \geq 8$.

\shortsection{Retraining Times}
Figure~\ref{fig:ablation retraining} shows that the more retraining times, the better performance our method can achieve, consistent with the result of \citet{souri2022sleeper}.
The number of retraining times is linearly related to the computational cost of adversarial learning, so excessive retraining times can become very expensive for the poisoner.

\section{Conclusion}

We introduced a more realistic threat model of G-TDP, generalizing from targeting a specific image to a set of varying physical objects. 
After analyzing the standard gradient matching algorithm, we proposed a novel adversarial loss that combines cosine similarity and Euclidean distance. 
Comprehensive experiments validated the superior generalization performance of our method.


{
    \small
    \bibliographystyle{ieeenat_fullname}
    \bibliography{main}
}

\appendix

\section{Smaller Gradient Magnitude Leads to More Generalizable Data Poisons}
\label{app:proposition}
We have claimed in Section~\ref{sec:motivation} that when the gradient of poison samples has a smaller Euclidean distance to the target gradient, the crafted data poisons can better generalize to semantically similar samples.
In this section, we provide a more formal proof.

First, noticing that under our G-TDP method, we regard the samples of the target physical object as a whole and let the poison samples match their average gradient.
That is, we poison an average target and hope the impact can be generalized to similar targets (other samples of the physical object).
Let $\bm{x}^t$ be the average target, and assume all samples of the physical object have a similar gradient (lie in an $L_2$ ball) to $\bm{x}^t$.
Let $\gL_{p, \mathrm{cos}}$ be the loss of poison samples optimized by $D_\mathrm{cos}$, and $\gL_{p, \mathrm{our}}$ be another loss of poison samples, which has a smaller Euclidean distance to the target gradient.
We wonder which condition of a sample $\bm{x}$ satisfies, the poison samples of $\gL_{p, \mathrm{our}}$ will influence it more significantly than the poison samples of $\gL_p$.

\begin{proposition}
Let $\theta$ be model parameters that satisfy
\begin{align*}
    & \cos(\nabla_\theta\gL_t, \nabla_\theta\gL_{p, \mathrm{our}}) = \cos(\nabla_\theta\gL_t, \nabla_\theta\gL_{p, \mathrm{cos}}) = c > 0, \\
    & \|\nabla_\theta\gL_t - \nabla_\theta\gL_{\mathrm{p, our}}\|_2 < \|\nabla_\theta\gL_t - \nabla_\theta\gL_{\mathrm{cos}}\|_2, \\
    & \|\nabla_\theta\gL_{p, \mathrm{cos}}\|_2 < c \cdot \|\nabla_\theta\gL_t\|_2.
\end{align*}
The first formula ensures that the two groups of poison samples match the gradients with the same cosine similarity. 
The second formula represents that the gradient of our poison samples has a smaller Euclidean distance to the target gradient.
The third formula holds since the magnitude of $\|\nabla_\theta\gL_t\|_2$ is large, but the magnitude of $\|\nabla_\theta\gL_{p, \mathrm{cos}}\|_2$ is much less than it, as we mentioned in Section~\ref{sec:motivation}.

If performing a single step of gradient descent with $\gL_p$
\begin{equation}
    \theta' = \theta - \eta \nabla_\theta \gL_p,
\end{equation}
then for all samples $\bm{x}$ that satisfies
\begin{equation}
    \label{eq:5}
    \begin{aligned}
    &\|\nabla_\theta \ell(f_\theta(\bm{x}), y_{\mathrm{adv}}) - \nabla_\theta \gL_t\|_2 < \\
    & \qquad \frac{\langle \nabla_\theta\gL_t, \nabla_\theta\gL_{\mathrm{p, our}} - \nabla_\theta\gL_{p, \mathrm{cos}}\rangle}{\|\nabla_\theta\gL_{\mathrm{p, our}} - \nabla_\theta\gL_{p, \mathrm{cos}}\|_2},
    \end{aligned}
\end{equation}
we have
\begin{equation}
    \label{eq:6}
    \ell(f_{\theta'_{\mathrm{our}}}(\bm{x}), y_{\mathrm{adv}}) < \ell(f_{\theta'_{\mathrm{cos}}}(\bm{x}), y_{\mathrm{adv}}).
\end{equation}
\end{proposition}

\begin{proof}
To simplify notations, we define:
\begin{align*}
&\bm{g}_{\mathrm{our}} = \nabla_\theta \mathcal{L}_{\mathrm{p, our}}, \quad \bm{g}_{\mathrm{cos}} = \nabla_\theta \mathcal{L}_{p, \mathrm{cos}}, \\
&\bm{g}_t = \nabla_\theta \mathcal{L}_t, \quad \bm{g}(\bm{x}) = \nabla_\theta\ell(f_\theta(\bm{x}), y_{\mathrm{adv}}).
\end{align*}
First, we show that Equation~\ref{eq:5} holds:
\begin{equation}
\begin{aligned}
\langle \bm{g}_t, \bm{g}_{\mathrm{our}} - \bm{g}_{\mathrm{cos}}\rangle &= \langle \bm{g}_t, \bm{g}_{\mathrm{our}}\rangle - \langle \bm{g}_t, \bm{g}_{\mathrm{cos}}\rangle \\
&= c \cdot \|\bm{g}_t\|_2 \cdot (\|\bm{g}_{\mathrm{our}}\|_2 - \|\bm{g}_{\mathrm{cos}}\|_2) \\
&> 0.
\end{aligned}
\end{equation}
We have $\|\bm{g}_{\mathrm{our}}\|_2 > \|\bm{g}_{\mathrm{cos}}\|_2$ due to the first three conditions in the proposition.
For any $\bm{x}$ satisfying Equation~\ref{eq:5}, expanding $\ell(f_{\theta'}(\bm{x}), y_{\mathrm{adv}})$ at $\theta'$ using Taylor's theorem gives:
\begin{equation}
\ell(f_{\theta'}(\bm{x}), y_{\mathrm{adv}}) - \ell(f_{\theta}(\bm{x}), y_{\mathrm{adv}}) = -\eta\langle \bm{g}(\bm{x}), \bm{g}\rangle + C,
\end{equation}
where $\bm{g}$ is  either $\bm{g}_{\mathrm{our}}$ or $\bm{g}_{\mathrm{cos}}$, and $C$ denotes the remainder term. To ensure $\ell(f_{\theta'}(\bm{x}), y_{\mathrm{adv}})$ decreases, $\langle \bm{g}(\bm{x}), \bm{g} \rangle$ must increase.
Decompose $\bm{g}(\bm{x})$ as $\bm{g}_t + \bm\delta$, where 
$$
\|\bm{\delta}\|_2 < \frac{\langle \bm{g}_t, \bm{g}_{\mathrm{our}} - \bm{g}_{\mathrm{cos}} \rangle}{\|\bm{g}_{\mathrm{our}} - \bm{g}_{\mathrm{cos}}\|_2}.
$$
Then, we have
\begin{equation}
\langle \bm{g}(\bm{x}), \bm{g} \rangle = \langle \bm{g}_t + \bm{\delta}, \bm{g} \rangle = \langle \bm{g}_t, \bm{g} \rangle + \langle \bm{\delta}, \bm{g} \rangle.
\end{equation}
To prove Equation~\ref{eq:6}, it suffices to show $\langle \bm{g}(\bm{x}), \bm{g}_{\mathrm{our}} \rangle > \langle \bm{g}(\bm{x}), \bm{g}_{\mathrm{cos}} \rangle$. Rearranging this inequality yields:
\begin{equation}
\label{eq:10}
\langle \bm{\delta}, \bm{g}_{\mathrm{our}} - \bm{g}_{\mathrm{cos}}\rangle > -\langle \bm{g}_t, \bm{g}_{\mathrm{our}} - \bm{g}_{\mathrm{cos}}\rangle.
\end{equation}
Applying the Cauchy-Schwarz inequality to $\langle \bm{\delta}, \bm{g}_{\mathrm{our}} - \bm{g}_{\mathrm{cos}}\rangle$ gives:
\begin{equation}
\label{eq:11}
\begin{aligned}
\langle \bm{\delta}, \bm{g}_{\mathrm{our}} - \bm{g}_{\mathrm{cos}}\rangle &\geq -\|\bm{\delta}\|_2 \cdot \|\bm{g}_{\mathrm{our}} - \bm{g}_{\mathrm{cos}}\|_2 \\
& > -\langle \bm{g}_t, \bm{g}_{\mathrm{our}} - \bm{g}_{\mathrm{cos}}\rangle,
\end{aligned}
\end{equation}
which confirms the lower bound in Equation~\ref{eq:10}.
\end{proof}

\begin{corollary}
When $\bm{g}_\mathrm{our}$ has a same direction to $\bm{g}_\mathrm{cos}$, the $L_2$ ball can be simplified to 
\begin{equation}
\|\bm{\delta}\|_2 < c \cdot \|\bm{g}_t\|_2.
\end{equation}
Since the $ \|\bm{g}_t\|_2$ is large (the target samples with incorrect labels) and $c$ is close to $1$, the impact of a smaller Euclidean distance can generalize to a very large area.
\end{corollary}

\section{Additional Experiments}

\subsection{Data Augmentation}
\label{app:augment}

\begin{figure}[t]
    \centering   
    \includegraphics[width=0.8\linewidth]{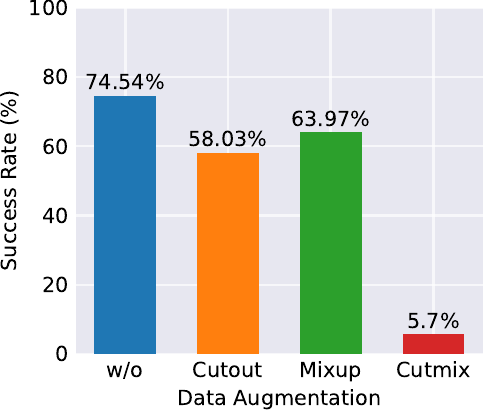}
    \caption{Poisoning success rates of standard gradient matching with/without applying data augmentation on target samples.}
    \label{fig:negative aug}
\end{figure}

To investigate whether data augmentation can simulate physical variations, we take advantage of the original method proposed in \cite{geiping2021witches}, and perform data augmentation on target samples.
According to Figure~\ref{fig:negative aug}, we find that the SR decreases.
We believe this result reveals the gap between the physical domain and the digital domain.
The poison samples are constructed with the variations in the digital domain, where these variations may not exist in the physical domain.
In other word, performing data augmentation on target samples makes a waste of the impact of poisoning on irrelevant variations, therefore leads to the decline of SR.

\subsection{Poisoning Transferability across Architectures}
\label{app:transfer}
Table~\ref{tab:blackbox_eval} reports the transfer poisoning success rate with different surrogate models.
The adversarial loss we adopted here is $D_{\mathrm{mul}}$.
As can be seen, the ResNet-18, which is commonly adopted \citep{geiping2021witches, schwarzschild2021just}, surprisingly achieves the worst transferability in our threat model.
In contrast, VGG11, which does not contain residual modules, shows the best results.

\begin{table}[t]
\centering
\caption{Transferability of poisoning SR (\%) on Multi-View Car. Each row represents a surrogate model for creating data poisons, while each column stands for a victim model.}
\label{tab:blackbox_eval}
\resizebox{\linewidth}{!}{
\begin{tabular}{l|cccc|c}
\toprule
\textbf{Model} & RN-18 & MobileNet & VGG11 & LeViT & Avg. \\
\midrule
RN-18      & $90.13$ & $40.78$ & $17.18$ & $10.10$ & $39.55$ \\
MobileNet & $60.68$ & $81.87$ & $30.60$ & $15.70$ & $47.21$ \\ 
VGG11     & $80.17$ & $74.42$ & $84.44$ & $61.09$ & $75.03$ \\
LeViT     & $32.89$ & $14.03$ & $57.32$ & $81.97$ & $46.55$ \\
\bottomrule
\end{tabular}
}
\end{table}

\begin{figure*}[tbp]
    \centering
    \includegraphics[width=0.8\textwidth]{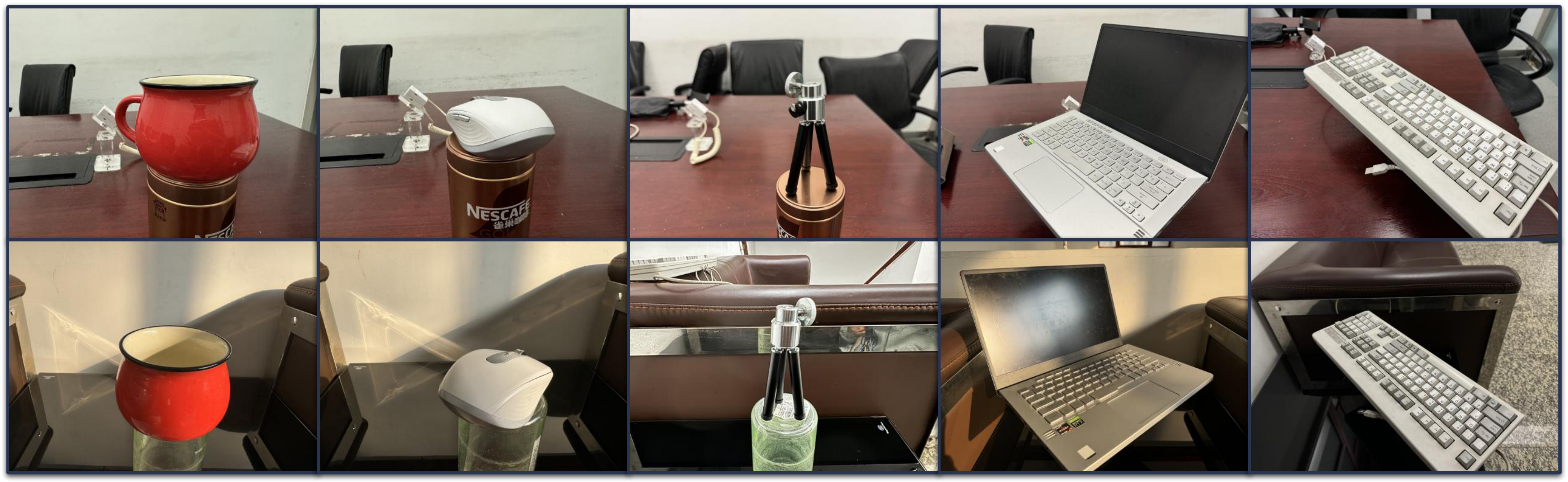}
    \caption{Visualization of target physical objects in our handmade dataset.}
    \label{fig:visualization appendix}
\end{figure*}

\begin{figure*}[tbp]
    \centering
    \adjustbox{valign=c}{
    \subfloat[Poison Samples with $\epsilon=16$ in CIFAR-10 dataset.]{
        \includegraphics[width=0.25\linewidth]{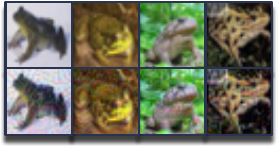}
    }
    }
    \adjustbox{valign=c}{
    \subfloat[Poison Samples with $\epsilon=8$ in ImageNet dataset.]{
        \includegraphics[width=0.7\linewidth]{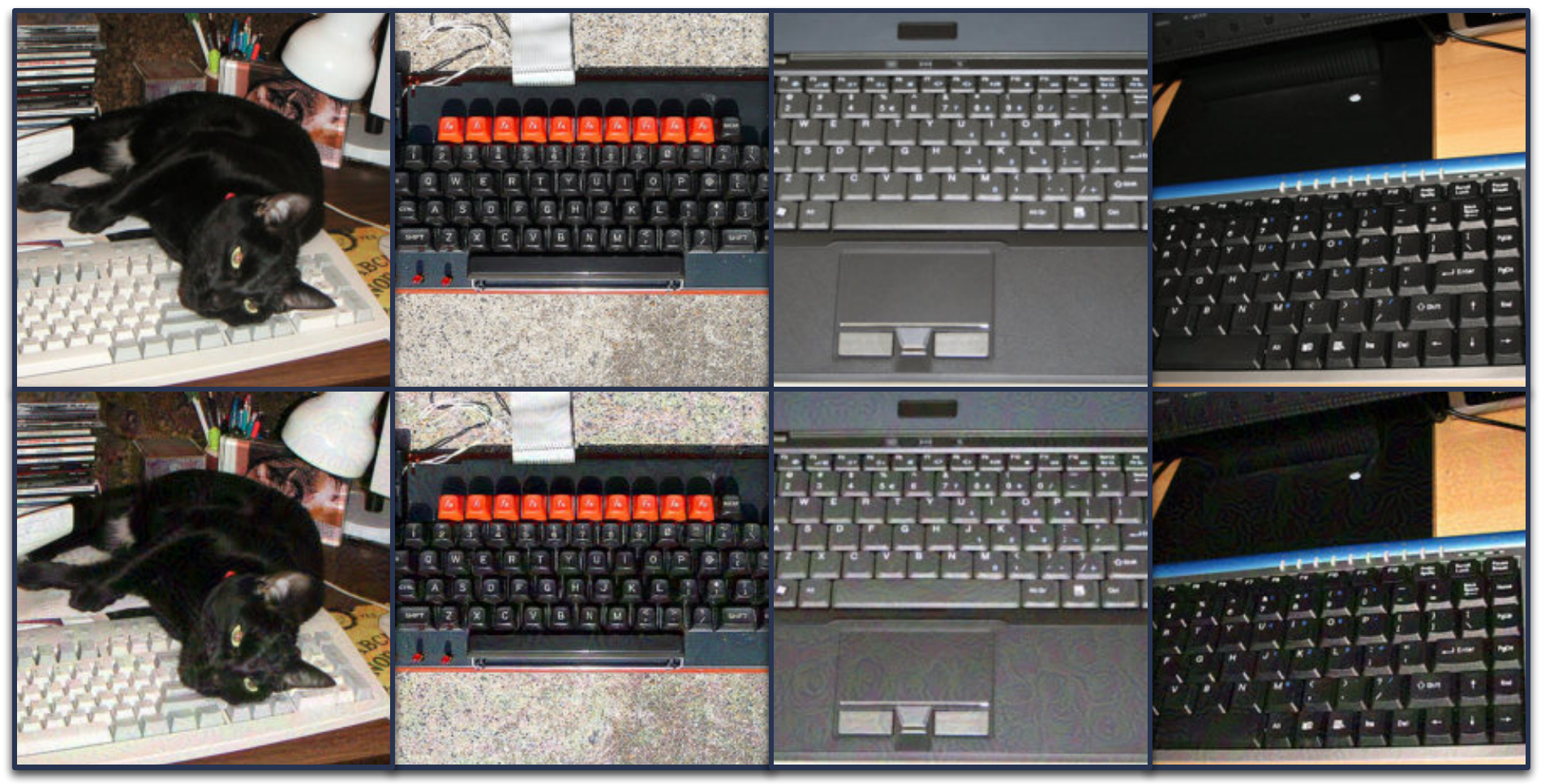}
    }
    }
    \caption{Visualization of poison samples.}
    \label{fig:poison}
\end{figure*}

\section{Additional Evaluation Details}
\label{app:experiment details}
\subsection{Main Experiments on Multi-View Car}

We consider $5$ neural network architectures: ConvNet, ResNet-18, VGG11, MobileNet V2, and LeViT-384.
The ConvNet contains $5$ convolutional layers followed by a linear layer.
A ReLU activation follows each convolutional layer, and for the last two convolutional layers, each layer is followed by a max pooling layer with size $3$.
The output widths of each layer are $64$, $128$, $128$, $256$, $256$, $2304$.
The ResNet-18 setting is followed by Witches' Brew \citep{geiping2021witches}, where they modify the first convolutional layer with a kernel size of $3$.
VGG-11, MobileNet V2, and LeViT-384 are with their standard configurations.

We set the learning rate to $0.1$ for ResNet-18 and $0.01$ for the others.
The batch size is $512$ for LeViT-384 and $128$ for the others.
We train ConvNet, ResNet-18, VGG-11, and MobileNet V2 for $40$ epochs and schedule the learning rate drops at $14$, $24$, and $35$ by a factor of $0.1$.
Since the ViTs need more epochs to converge, the total number of training epochs of LeViT-384 is set to $100$.
We adjust the learning rate, dropping at $37$, $62$, and $87$ by a factor of $0.1$.
We apply the differential data augmentation that horizontally flips the image with a probability of $0.5$ and randomly crops a size $32 \times 32$ with a zero-padding of $4$.

The Swin Transformer we discussed in Section~\ref{sec:exp results} is a modified Swin V2 Tiny~\cite{liu2022swin} model that reduces the patch size from $4$ to $2$ and the window size from $8$ to $4$.
We train this model with a learning rate of $0.01$ for $100$ epochs and drop the learning rate at $50$, $75$, and $90$ by a factor of $0.1$.

\subsection{More Details about ImageNet Experiments}

For ImageNet, we only use $5$ different cars as target objects. 
This is because ImageNet contains $1,000$ categories of images, and there are some overlaps (such as convertible and sports car).
The different views of some vehicles can even be classified into distinct categories.
For simplicity, we select $5$ cars, which will all be classified into the ``sports car''  category regardless of their viewpoints. 

For the single trial in the complete ImageNet, the hyperparameter is the same as in the ImageNet subset, but differs in $R=1$.
The model architecture is ResNet-18.
To evaluate the poisoning more realistically, the poisoner trains the surrogate model with $60$ epochs, but the victim trains the model with $120$ epochs.
The success of data poisoning is usually achieved as long as the poisoned samples are constructed on a converged model; different training settings will not affect their effectiveness.

\section{Visualizations}
\label{app:visualization}
\subsection{Target Variations in Handmade Dataset}

There are $5$ classes of targets we have discussed in the corresponding section in the paper: cup, computer mouse, tripod, laptop, and keyboard.
Figure~\ref{fig:visualization appendix} shows our chosen real-world objects.

\subsection{Generated Poison Samples}

We visualize a couple of images of the poison samples generated by our method in Figure~\ref{fig:poison}.

\end{document}